%% file: arxiv_main.tex
\documentclass[11pt]{article}
\pdfoutput=1

\usepackage[vmargin=1.00in,hmargin=1.00in,centering,letterpaper]{geometry}

\usepackage[utf8]{inputenc} % allow utf-8 input
\usepackage[T1]{fontenc}    % use 8-bit T1 fonts
\usepackage{hyperref}       % hyperlinks
\usepackage{url}            % simple URL typesetting
\usepackage{booktabs}       % professional-quality tables
\usepackage{amsfonts}       % blackboard math symbols
\usepackage{nicefrac}       % compact symbols for 1/2, etc.
\usepackage{microtype}      % microtypography

\usepackage{algorithm} 
\usepackage[noend]{algpseudocode}

\usepackage{amsmath,amsthm,amsfonts}
\usepackage{color}
\usepackage{mathrsfs}
\usepackage{enumitem} 
\usepackage{bm}
\usepackage{graphicx}
\usepackage{dsfont}
\graphicspath{{./figures/}}
\usepackage[square, numbers]{natbib}
\usepackage[labelfont=bf]{subcaption}

\include{notation}
\newtheorem{theorem}{Theorem}
\newtheorem{lemma}[theorem]{Lemma}

\newtheorem{example}[theorem]{Example}
\newtheorem{corollary}[theorem]{Corollary}

\newtheorem{proposition}[theorem]{Proposition}
\theoremstyle{definition}

\newcommand{\EE}{\mathbb{E}}
\newcommand{\Ocal}{\mathcal{O}}

\algrenewcommand\algorithmicrequire{\textbf{input:}}

\DeclareCaptionSubType*{algorithm}

\DeclareCaptionLabelFormat{alglabel}{Alg.~#2}

\def \basefigwidth{0.32}
\begin{document}
  
\title{\textbf{Competing Bandits in Matching Markets}}

\author{Lydia T. Liu~\thanks{Equal contribution} \quad \quad  Horia Mania~\footnotemark[1] \quad \quad   Michael I. Jordan\\
\\
Department of Electrical Engineering and Computer Sciences\\
University of California, Berkeley}

%\author{%
%	David S.~Hippocampus\thanks{Use footnote for providing further information
%		about author (webpage, alternative address)---\emph{not} for acknowledging
%		funding agencies.} \\
%	Department of Computer Science\\
%	Cranberry-Lemon University\\
%	Pittsburgh, PA 15213 \\
%	\texttt{hippo@cs.cranberry-lemon.edu} \\
	% examples of more authors
	% \And
	% Coauthor \\
	% Affiliation \\
	% Address \\
	% \texttt{email} \\
	% \AND
	% Coauthor \\
	% Affiliation \\
	% Address \\
	% \texttt{email} \\
	% \And
	% Coauthor \\
	% Affiliation \\
	% Address \\
	% \texttt{email} \\
	% \And
	% Coauthor \\
	% Affiliation \\
	% Address \\
	% \texttt{email} \\
%}

\date{}

\maketitle

\begin{abstract}
	\input{abstract}
  %Two-sided markets are prevalent in economics and stable matchings are natural equilibria when each side of a two-sided market has preferences over the other side. We study two-sided markets in which one side of the market does not have a priori knowledge about its utility for the other side and is required to learn its utilities statistically. Our model extends the standard stochastic multi armed bandits framework to multiple agents in a new way, where agents represent one side of the market and arms represent the other side. We study both centralized and decentralized approaches to this problem and show surprising exploration-exploitation trade-offs compared to the single agent multi-armed bandits setting. 
\end{abstract}

\section{Introduction}

Research in machine learning has focused on pattern recognition in recent years.
Another major branch of machine learning---decision making under uncertainty---has comparatively received 
less attention.  Literatures on bandits and reinforcement learning, though active, mainly study problem settings involving single decisions or 
decision-makers.  In real-world settings, individual decisions must generally 
be made in the context of other related decisions.  Moreover, real-world decisions 
often involve scarcity, with competition among multiple decision-makers. This 
yields additional uncertainty and complications, implying there are tradeoffs to be 
sought. To study such settings we need to blend economics with learning.

In this paper, we present a formal study of such a blend.  We focus on the multi-arm
bandit (MAB) problem, a core machine-learning problem in which there are $K$ actions giving stochastic rewards, and 
the learner must discover which action gives maximal expected reward \citep{Bubeck12regret,  Lai1985asymp, Szepesvari2018algs, Thompson33}.
The bandit problem highlights the fundamental tradeoff between exploration and exploitation.  
Regret bounds quantify this tradeoff.  We study an economic version of the 
problem in which there are \emph{multiple} agents solving a bandit problem, and
there is competition---if two or more agents pick the same arm, only one of the agents 
is given a reward.\footnote{Note that \cite{MansourSW18} and \cite{Aridor2019competing} have used the term ``competing bandits'' for a  different problem formulation where a user can choose between two different bandit algorithms; this differs from our setting where multiple learners compete over scarce resources.} 
We assume that the arms have a preference ordering over 
the agents---a key 
point of departure from the line of work on multi-player bandits with 
collisions \citep{bubeck2019multiplayer, pmlr-v49-cesa-bianchi16, Liu10distributed, Shahrampour17}---and this ordering is unknown a priori to the agents.

We are motivated by problems involving two-sided markets that link producers 
and consumers or workers and employers,
where each side sees the other side via a recommendation system, and where 
there is scarcity on the supply side (for example, a restaurant has a limited number of 
seats, a street has a limited capacity, or a worker can attend to one task at a time).
The overall goal is an economic one---we 
wish to find a stable matching between producers and consumers.  To study the core 
mathematical problems that arise in such a setting, we have abstracted away the 
recommendation systems on the two sides, modeling them via the preference orderings 
and the differing reward functions. Several massive online labor and service markets can be captured by this abstraction; see the end of this section for an illustration of an  application.
In the context of two-sided markets the arms' preferences can be explict, e.g. when the arms represent entities in the market with their own utilities for the other side of the market, or implicit, e.g. when the arms represent resources their ``preferences'' encode the skill levels of the agents in securing those resources.

To determine the appropriate notions of equilibria in our multi-agent MAB model, we turn to the literature on stable matching in two-sided markets \cite{gale62college, Gusfield1989stable,roth1990two,knuth97stable, Roth2008}. Since its introduction by \citet{gale62college}, the stable matching problem has had high practical impact, leading to improved matching systems for high-school admissions and labor markets \cite{roth1984evolution}, house allocations with existing tenants \cite{abdulkadirouglu1999house}, content delivery networks \cite{maggs2015algorithmic}, and kidney exchanges \cite{roth2005pairwise}. 

In spite of these advances, standard matching models tend to assume that entities in the market know their preferences over the other side of the market. Models that allow unknown preferences usually assume that preferences can be discovered through one or few interactions \cite{Ashlagi2017communication}, e.g., one interview per candidate in the case of medical residents market \cite{roth1984evolution, roth1990two}. These assumptions do not capture the statistical uncertainty inherent in problems where data informs preferences. 

In contrast, our work is motivated by modern matching markets which operate at scale and require repeated interactions between the two sides of the market, leading to exploration-exploitation tradeoffs. We consider two-sided markets in which entities on one side of the market do not know their preferences over the other side, and develop matching and learning algorithms that can provably attain a stable market outcome in this setting. Our contributions are as follows: 
  
\begin{itemize}[leftmargin=*]
\item We introduce a new model for understanding two-sided markets in which one side of the market does not know its preferences over the other side, but is allowed multiple rounds of interaction. Our model combines work on multi-armed bandits with work on stable matchings. In particular, we define two natural notions of regret, based on stable matchings of the market, which quantify the exploration-exploitation trade-off for each individual agent. 
\item We extend the Explore-then-Commit (ETC) algorithm for single agent MAB to our multi-agent setting. We consider two versions of ETC: centralized and decentralized. For both versions we prove $\Ocal(\log(n))$ problem-dependent upper bounds on the regret of each agent.
 \item In addition to the known limitations of ETC for single agent MAB, in Section~\ref{sec:centralized_ucb} we discuss other issues with ETC in the multi-agent setting. To address these issues we introduce a centralized version of the well-known upper confidence bound (UCB) algorithm. We prove that centralized UCB achieves $\Ocal(\log(n))$ problem-dependent upper bounds on the regret of each agent. Moreover, we show that centralized UCB is incentive compatible.
\end{itemize}
 %tability is a very useful notion of efficiency in a two-sided market. In the machine learning community, bandits algorithms have been studied extensively to optimize the exploration-exploitation trade-off, but largely without considering interactions between multiple learning agents and questions of stability.
 
 %Multi armed bandits is the simplest reinforcement learning problem in which exploration-exploitation tradeoffs are important. Recently there has been a lot of interest in multi-agent RL. We expand on previous work on multi-agent multi-armed bandits to consider a competitive setting with natural notations of regret in terms of stable matches. We aim to understand the exploration-exploitation tradeoffs. 
 
 % We are interested in two-sided markets with stochastic utilities of unknown means. In particular, there is information asymmetry where one side of market initiates the exploration (agent side has uncertainty).
 %The preferences of the arms can be understood as actual preferences of the entities on that side of the market, or as a measure of skill of the agents for different tasks (represented by the arms), or as a measure of the amount of money the agents are willing to pay for securing that arm.
 
Most of the above results can be extended to the case where arms also have uncertain preferences over agents in a straightforward manner. For the sake of simplicity, we focus on the setting where one side of market initiates the exploration and leave extensions of our results to future work.

\paragraph{Online labor markets} Our model is applicable to matching problems that arise in online labor markets (e.g., Upwork and Taskrabbit for freelancing, Handy for housecleaning) and online crowdsourcing platforms (e.g., Amazon Mechanical Turks). In this case, the employers, each with a stream of similar tasks to be delegated, can be modeled as the players, and the workers can be modeled as the arms. For an employer, the mean reward received from each worker when a task is completed corresponds to how well the task was completed (e.g., did the Turker label the picture correctly?). This differs for each worker due to differing skill levels, which the employer does not know a priori and must learn by exploring different workers. A worker has preferences over different types of tasks (e.g., based on payment or prior familiarity the task) and can only work on one task at a time; hence they will pick their most preferred task to complete out of all the tasks that are offered to them.

\section{Problem setting}

We denote the set of $N$ agents by $\players = \{\player{1}, \player{2}, \ldots, \player{\numagents}\}$ 
and the set of $K$ arms by $\arms = \{ \arm{1}, \arm{2}, \ldots, \arm{\numarms}\}$. We assume $N \leq K$.
At time step $t$, each agent $\player{i}$ selects an arm $\action{t}{i}$, where 
$\actions{t} \in \arms^\numagents$ is the vector of all agents' selections.

When multiple agents select the same arm only one agent is allowed to pull the arm, according to the 
arm's preferences via a mechanism we detail shortly. Then, if player $\player{i}$ successfully pulls 
arm $\action{t}{i}$ at time $t$, they are said to be \emph{matched} to $\action{t}{i}$ at time $t$ and they receive a stochastic reward $\reward{i}{\actions{t}}{t}$ sampled from a $1$-sub-Gaussian 
distribution with mean $\meanreward{i}{\action{t}{i}}$.
%For each fixed agent $\player{i}$ we assume $\meanreward{i}{j} \neq \meanreward{i}{j^\prime}$ for all distinct arms $\arm{j}$ and $\arm{j^\prime}$, that is, preferences are strict.

Each arm $\arm{j}$ has a fixed known ranking $\prefsarm{j}$ of the agents, where $\prefarm{j}{i}$ is the rank of player $\player{i}$. In other words, $\prefsarm{j}$ is a permutation of $[\numagents]$ and $\prefarm{j}{i} < \prefarm{j}{i^{\prime}}$ implies that arm $\arm{j}$ prefers player $\player{i}$ to player $\player{i^{\prime}}$.
If two or more agents attempt to pull the same arm $\arm{j}$, there is a \emph{conflict} and only the top-ranked agent successfully pulls the arm to receive a reward; the other agent(s) $\player{i^\prime}$ is said to be \emph{unmatched} and does not receive any reward, that is, $\reward{i^\prime}{\actions{t}}{t} = 0$. As a shorthand, the notation $\player{i} \succ_{j} \player{i^\prime}$ means that arm $\arm{j}$ prefers player $\player{i}$ over $\player{i^\prime}$. When arm $\arm{j}$ is clear from context, we simply write $\player{i} \succ \player{i^\prime}$. Similarly, the notation  $\arm{j} \succ_{i} \arm{j^\prime}$ means that $\player{i}$ prefers arm $\arm{j}$ over $\arm{j^\prime}$, i.e. $\meanreward{i}{j} > \meanreward{i}{j^\prime}$.

Given the full preference rankings of the arms and players, arm $\arm{j}$ is called a \emph{valid match} of player $\player{i}$ if there exists a stable matching according to those rankings such that $\arm{j}$ and $\player{i}$ are matched. We say $\arm{j}$ is the \emph{optimal match} of agent $\player{i}$ if it is the most preferred valid match. Similarly, we say $\arm{j}$ is the \emph{pessimal match} of agent $\player{i}$ if it is the least preferred valid match. 
Given complete preferences, the Gale-Shapley (GS) algorithm \citep{gale62college} finds a stable matching after repeated proposals from one side of the market to the other. The matching returned by the GS algorithm is always optimal for each member of the proposing side and pessimal for each member of the non-proposing side \citep{knuth97stable}. 

We denote by $\optmatches$ and $\pessmatches$ the functions from $\players$ to $\arms$ that define the optimal and pessimal matchings  of the players according to the true preferences of the players and arms. Then, it is natural to define the \emph{agent-optimal stable regret} of agent $\player{i}$ as 
\begin{align}
\label{eq:opt_regret} 
\optregret{i}(\horizon) := \horizon \meanreward{i}{\optmatch{i}} - \sum_{t = 1}^\horizon \EE \reward{i}{\actions{t}}{t},
\end{align}
because when the arms' mean rewards are known the GS algorithm outputs the optimal matching $\optmatches$, and in online learning, regret is generally defined so that the reward of the agent is as good as the reward of playing the best action in hindsight at every time step. However, as we show in the sequel, there is a desirable class of centralized algorithms which cannot achieve sublinear agent-optimal stable regret. Therefore, we also consider the \emph{agent-pessimal stable regret} defined by 
\begin{align}
\label{eq:pess_regret}
\pessregret{i}(\horizon) := \horizon \meanreward{i}{\pessmatch{i}} - \sum_{t = 1}^\horizon \EE \reward{i}{\actions{t}}{t}.
\end{align}

Throughout we assume that the agents cannot observe each other's rewards or confidence intervals for the arms' mean rewards. Now, we detail several interaction settings which are of interest:

\paragraph{Centralized:} At each time step the agents are required to send a ranking of the arms to a matching platform. Then, the platform decides the action vector $\actions{t}$. In this work we consider two platforms. The first platform (shown in on the left of Table~\ref{table:algs}) outputs a random assignment for a number of time steps and then computes the agent-optimal stable matching  according to the agents' preferences. The second platform (shown on the right of Table~\ref{table:algs}) takes in the agent's preferences at each time step and outputs a stable matching between the agents and arms. Both platforms ensure that there will be no conflicts between the agents. The first platform corresponds to an explore-then-commit strategy. When the second platform is used the agents must rank arms in a way which enables exploration and exploitation. We show that ranking according to upper confidence bounds yields $\Ocal(\log(\horizon))$ agent-pessimal stable regret.

\paragraph{Decentralized with partial information:} Agents observe each other's actions and the outcomes of the ensuing conflicts, but do not have a medium for coordination and communication. 

\paragraph{Decentralized with no information:} After selecting an arm, agents observe whether they lost a conflict on that arm. When they successfully pull an arm they observe their own reward. However, players do not observe any other information and do not have access to a medium for coordination and communication. In Section \ref{sec:decent-etc}, we analyze an explore-then-commit scheme in this setting.

\section{Multi-agent bandits with a platform}

%First we analyze two algorithms which rely on a platform for coordination. It is natural to design two-sided markets around a platform when conflicts between agents are costly. 
  
\subsection{Centralized Explore-then-Commit}
\label{sec:centralized_etc}

In this section we give a guarantee for the explore-then-commit planner defined in Algorithm~\ref{table:algs}(left). At each iteration, each agent $\player{i}$ updates their mean reward for arm $j$ to be  
\begin{align}
\label{eq:empirical_mean}
\hatmeanreward{i}{j}{t} = \frac{1}{T_{i,j}(t)}  \sum_{s = 1}^t \indi\{\action{s}{i} = j\} \reward{i}{\actions{s}}{s},
\end{align}
where $T_{i,j}(t) = \sum_{s = 1}^t \indi\{\action{t}{i} = j\}$ is the number of times agent $\player{i}$ successfully pulled arm $\arm{j}$. At each time step, player $\player{i}$ ranks the arms in decreasing order according to $\hatmeanreward{i}{j}{t}$ and sends the resulting ranking $\prefsagent{i}{t}$ to the platform. As seen in Table~\ref{table:algs}, for the first $h \numarms$ time steps, the platform assigns players to arms cyclically, ensuring that each agent samples every arm $h$ times. We now provide a regret analysis of centralized ETC. The proof is deferred to Section~\ref{sec:proof-etc}.
 
\begin{table}
  \begin{tabular}{l|r}
    \hline
\begin{subalgorithm}[t]{.47\textwidth}
\begin{algorithmic}[1]
  \Require{$h$, and the preference ranking $\prefsarm{j}$ of all arms $\arm{j} \in \arms$, the horizon length $\horizon$}
%  \Procedure{Platform}{}
  \For{$t = 1, \ldots, T $}
  \If{$t \leq h \numarms$}
      \State $\action{t}{i} \gets \arm{t + i - 1 \pmod{ \numarms} + 1}$,  $\forall i$.
  \ElsIf{$t = h \numarms + 1$}
  \State Receive rankings $\prefsagent{i}{t}$ from all $\player{i}$.
  \State Compute agent-optimal stable matching $\action{t}{i}$ according to $\prefsagent{i}{t}$ and $\prefsarm{j}$.
  \Else 
  \State $\action{t}{i} \gets \action{h\numarms + 1}{i}$,  $\forall i$.  
  \EndIf
  \EndFor
%  \EndProcedure 
\end{algorithmic}
\end{subalgorithm} &
\begin{subalgorithm}[t]{.47\textwidth}
\begin{algorithmic}[1]
  \Require{the preference ranking $\prefsarm{j}$ of all arms $\arm{j} \in \arms$}
  %\Procedure{Planner}{}
  \For{$t = 1, \ldots, T $}
  \State Receive rankings $\prefsagent{i}{t}$ from all $\player{i}$.
  \State Compute agent-optimal stable matching $\actions{t}$ according to all $\prefsagent{i}{t}$ and $\prefsarm{j}$.
  \EndFor
  %\EndProcedure 
\end{algorithmic}
\end{subalgorithm} \\  \\ \hline 
\end{tabular}
\caption{(\emph{left}) Explore-then-Commit Platform. (\emph{right)} Gale-Shapley Platform.}
\label{table:algs}
\end{table}

\begin{theorem}
  \label{thm:centralized_etc}
  Suppose all players rank arms according to the empirical mean rewards \eqref{eq:empirical_mean} and submit their rankings to the explore-then-commit platform. Let $\optgap{i,j} = \meanreward{i}{\optmatch{i}} - \meanreward{i}{j}$, $\optgap{i, \max} = \max_{j} \optgap{i,j}$, and  $\Delta = \min_{i \in [\numagents]} \min_{j \colon \optgap{i,j} > 0} \optgap{i,j} > 0$. Then, the expected agent-optimal regret of player $\player{i}$ is upper bounded by
  \begin{align}
    \optregret{i}(\horizon) \leq h \sum_{j = 1}^\numarms \optgap{i,j} + (\horizon - h \numarms) \optgap{i, \max} \numagents \numarms \exp\left(- \frac{h \Delta^2}{4}\right). 
  \end{align}
  In particular, if  $h = \max \left\{1, \frac{4}{\Delta^2} \log \left(1 +  \frac{n \Delta^2 N }{4}\right) \right\}$, we have
\begin{align}
  \optregret{i}(\horizon) \leq  \max \left\{1,   \frac{4}{\Delta^2} \log \left(1 +  \frac{n \Delta^2 N }{4}\right) \right\} \sum_{j = 1}^\numarms \optgap{i,j} + \frac{4 K \optgap{i, \max}}{\Delta^2} \log \left(1 +  \frac{n \Delta^2 N}{4}\right).
\end{align}
\end{theorem}
  
This result shows that centralized ETC achieves $\Ocal(\log(n))$ agent-optimal stable regret when the number of exploration rounds is chosen apriopriately. As is the case for single agent ETC, centralized ETC requires knowledge of both the horizon $\horizon$ and the minimum gap $\Delta$ \citep[see, e.g.,][Chapter 6]{Szepesvari2018algs}. However, a glaring difference between the the settings is that in the latter the regret of each agent scales with $1/ \Delta^2$, where $\Delta$ is the minimum reward gap between the optimal match and a suboptimal arm across all agents. In other words, the regret of an agent might depend on the suboptimality gap of other agents. Example~\ref{example:delta2} shows that this dependence is real in general and not an artifact of our analysis. Moreover, while single agent ETC achieves $\Ocal(\sqrt{\horizon})$ problem-independent regret, Example~\ref{example:delta2} shows that centralized ETC does not have this desirable property. Finally, $\sum_{j = 1}^K \optgap{i,j}$ could be negative for some agents. Therefore, some agents can have negative agent-optimal regret, an effect that never occurs in the single agent MAB problem. 
 
\begin{example}[The dependence on $1/\Delta^2$ cannot be improved in general] 
  \label{example:delta2}
  Let $\players = \{\player{1}, \player{2}\}$ and $\arms = \{\arm{1}, \arm{2} \}$ with true preferences:
  \begin{align*}
    &\player{1} \colon \arm{1} \succ \arm{2} \quad &\arm{1} \colon \player{1} \succ \player{2}\\
    &\player{2} \colon \arm{2} \succ \arm{1} \quad &\arm{2} \colon \player{1} \succ \player{2}. 
  \end{align*}
  The agent-optimal stable matching is given by $\optmatch{1} = 1$ and $\optmatch{2} = 2$. Both $\arm{1}$ and $\arm{2}$ prefer $\player{1}$ over $\player{2}$. Therefore, at the end of the exploration stage $\player{1}$ is matched to their top choice arm while $\player{2}$  is  matched to the remaining arm. In order for $\player{2}$ to be matched to their optimal arm, $\player{1}$ must correctly determine that they prefer $\arm{2}$ over $\arm{1}$. The number of exploration rounds would then have to be $\Omega(1/\optgap{1, 2}^2)$ where $\optgap{1,2} = \meanreward{1}{2} - \meanreward{1}{2}$. Hence, when $\optgap{1,2} \leq 1/ \sqrt{\horizon}$, the regret of $\player{2}$ is $\Omega(\horizon \optgap{2,1})$. Figure~\ref{fig:gap} depicts this effect empirically; we observe that a smaller gap $\optgap{1,2}$ causes $\player{1}$ to have larger regret. 
\end{example}

	\subsubsection{Proof of Theorem \ref{thm:centralized_etc}}\label{sec:proof-etc}

First we present two instructive lemmas that are used in the proof of Theorem~\ref{thm:centralized_etc},
Throughout the remainder of this section, we say the ranking $\prefsagent{i}{t}$ submitted by $\player{i}$ at time $t$ is \emph{valid} if whenever an arm $a_j$ is ranked higher than $\optmatch{i}$, i.e.  $\prefagent{i}{j}{t} < \prefagent{i}{\optmatch{i}}{t}$, it follows that $\meanreward{i}{j} > \meanreward{i}{\optmatch{i}}$. 

\begin{lemma}
	\label{lem:opt_valid_rankings}
	If all the agents submit valid rankings to the planner, then the GS-algorithm finds a match $m$ such that   $\meanreward{i}{m(i)} \geq \meanreward{i}{\optmatch{i}}$ for all players $\player{i}$.
\end{lemma}
\begin{proof}
	First we show that true agent optimal matching $\overline{m}$ is stable according to the rankings submitted by the agents when all those rankings are valid. Let $\arm{j}$ be an arm such that $\prefagent{i}{j}{t}  < \prefagent{i}{\optmatch{i}}{t}$ for an agent $\player{i}$. Since $\prefsagent{i}{t}$ is valid, it means $\player{i}$ prefers $\arm{j}$ over $\optmatch{i}$ according to the true preferences also. However, since $\overline{m}$ is stable according to the true preferences, arm $\arm{j}$ must prefer player $\overline{m}^{-1}(j)$ over $\player{i}$, where $\overline{m}^{-1}(j)$ is $\arm{j}$'s match according to $\overline{m}$ or the emptyset if $\arm{j}$ does not have a match. Therefore, according to the ranking $\prefsagent{i}{t}$, $\player{i}$ has no incentive to deviate to arm $\arm{j}$ because that arm would reject her.
	Now, since $\overline{m}$ is stable according to the rankings $\prefsagent{i}{t}$, we know that the GS-algorithm will output a matching which is at least as good as $\overline{m}$ for all agents according to the rankings $\prefsagent{i}{t}$. Since all the rankings are valid, it follows that the GS-algorithm will output a matching $m$ which is as least as good as $\overline{m}$ according to the true preferences also, i.e., $\meanreward{i}{m(i)} > \meanreward{i}{\optmatch{i}}$.  
\end{proof}

\begin{lemma}
	\label{lem:prob_invalid}
	Consider the agent $\player{i}$ and let $\optgap{i,j} = \meanreward{i}{\optmatch{i}} - \meanreward{i}{j}$ and $\optgap{i, \min} = \min_{j \colon \optgap{i,j} > 0} \optgap{i,j}$.   Then, if $\player{i}$ follows the Explore-then-Commit platform (see Table~\ref{table:algs}(a)), we have
	\begin{align*}
	\PP(\prefsagent{i}{h \numarms} \text{ is invalid }) \leq \numarms e^{- \frac{h \optgap{i,\min}^2}{2}}. 
	\end{align*}
\end{lemma}
\begin{proof}
	Throughout this proof we denote $t = h\numarms$ as a shorthand. 
	In order for the ranking $\prefsagent{i}{t}$ to not be valid there must exist an arm $\arm{j}$ such that $\meanreward{i}{\optmatch{i}} > \meanreward{i}{j}$, but $\prefagent{i}{j}{t} < \prefagent{i}{\optmatch{i}}{t}$. This can happen only when $\hatmeanreward{i}{j}{t} \geq \hatmeanreward{i}{\optmatch{i}}{t}$. The probability of this event is equal to 
	\begin{align*}
	\PP\left( \hatmeanreward{i}{j}{t} \geq \hatmeanreward{i}{\optmatch{i}}{t} \right) &= \PP\left( \hatmeanreward{i}{\optmatch{i}}{t} - \meanreward{i}{\optmatch{i}} -  \hatmeanreward{i}{j}{t} + \meanreward{i}{j} \leq \meanreward{i}{j} - \meanreward{i}{\optmatch{i}} \right) \\
	&\leq \PP\left( \hatmeanreward{i}{\optmatch{i}}{t} - \meanreward{i}{\optmatch{i}} -  \hatmeanreward{i}{j}{t} + \meanreward{i}{j} \leq \optgap{i, \min} \right).
	\end{align*}
	Since each agent pulls each arm exactly $h$ times during the exploration stage and since the rewards from each arm are $1$-sub-Gassian, we know that $\hatmeanreward{i}{j^\prime}{t} - \meanreward{i}{j^\prime} -  \hatmeanreward{i}{j}{t} + \meanreward{i}{j}$
	is $\sqrt{2 / h}$-sub-Gaussian. Therefore,
	\begin{align*}
	\PP\left( \hatmeanreward{i}{j}{t} \geq \hatmeanreward{i}{\optmatch{i}}{t} \right) &\leq e^{-\frac{h \Delta_i^2}{4}}. 
	\end{align*}
	The conclusion follows by a union bound over all possible arms $\arm{j}$.
\end{proof}

\begin{proof}[Proof of Theorem~\ref{thm:centralized_etc}]
	During the exploration stage each player $\player{i}$ pulls each arm $\arm{j}$ exactly $h$ times. Therefore, the expected agent-optimal stable regret of agent $\player{i}$ after the first $h\numarms$ time steps is exactly equal to $h \sum_{j = 1}^\numarms \optgap{i,j}$ (note that $\optgap{i,j}$ might be negative for some values of $j$). The agent-optimal stable regret $\player{i}$ from time $h\numarms + 1$ to time $\horizon$ is at most $(\horizon - h \numarms)\optgap{i, \max}$. However, from Lemma~\ref{lem:opt_valid_rankings} we know that $\player{i}$ can incurr positive regret only if there exists a player who submits an invalid ranking at time $h \numarms + 1$. Lemma~\ref{lem:prob_invalid}, together with a union bound over all agents, ensures that the probability there exists a player who submits an invalid ranking is at most $N \exp \left(- \frac{h \Delta^2}{4}\right)$. This completes the proof. 
\end{proof}

\subsection{Centralized UCB}
\label{sec:centralized_ucb}

In the previous section we saw that centralized ETC achieves $\Ocal(\log(\horizon))$ agent-optimal regret for all agents. However, centralized ETC must know the horizon $\horizon$ and the minimum gap $\Delta$ between an optimal arm and a suboptimal arm. While knowing the horizon $\horizon$ is feasible in certain scenarios, knowing $\Delta$ is not plausible. It is known that single agent ETC achieves $\Ocal(\horizon^{2/3})$ when the number of exploration rounds is chosen deterministically without knowing $\Delta$, and there are also known methods for adaptively choosing the number of exploration rounds so that single agent ETC achieves $\Ocal(\log(n))$ \cite{Szepesvari2018algs}. However, in our setting, the $\Ocal(\horizon^{2/3})$ guarantee does not hold because the suboptimality gaps of one agent affect the regret of other agents, and the known adaptive stopping times cannot be implemented because the platform does not observe the agents' rewards. Therefore, it is necessary to find methods which do not need to know $\Delta$. 
%We consider an extension of UCB methods which are a popular approach to resolving these challenges in single agent MABs (see Related work).

Another drawback of centralized ETC is that it requires agents to learn concurrently. It thus does not take prior knowledge of preferences into account and forces that player to explore arms which might be suboptimal for them. The Gale-Shapley Platform shown in Table~\ref{table:algs}(right) resolves this problem, always outputting the agent-optimal matching given the rankings received from the agents. We derive an upper bound on the regret in this setting when all agents use upper confidence bounds to rank arms. In Section~\ref{sec:honesty} we show this method is incentive compatible.

Before proceeding with the analysis we define more precisely the UCB method employed by each agent and also introduce several technical concepts. 
At each time step the platform matches agent $\player{i}$ with arm $\action{t}{i}$. Each player $\player{i}$ successfully pulls arm $\action{t}{i}$, receives reward $\reward{i}{\actions{t}}{t}$, and updates their empircal mean for $\action{t}{i}$ as in \eqref{eq:empirical_mean}. They then compute the upper confidence bound
\begin{align}
 \label{eq:ucb}
  \ucb{i}{j}{t} = \left\{
  \begin{array}{ll}
    \infty & \text{if } T_{i,j}(t) = 0,\\
    \hatmeanreward{i}{j}{t} + %\sqrt{\frac{4\log(n)}{T_{i,j}(t)}} 
    \sqrt{\frac{3\log t}{2T_i(t-1)}} & \text{otherwise.} 
   \end{array}\right.
\end{align}
Finally, each player $\player{i}$  orders the arms according to  $\ucb{i}{j}{t}$ and computes the ranking $\prefsagent{i}{t + 1}$ so that a higher upper confidence bound means a better rank, e.g. $\arg\max_j \ucb{i}{j}{t}$ is ranked first in $\prefsagent{i}{t + 1}$. 

Let $m$ be an injective function from the set of players $\players$  to the set of arms $\arms$; hence $m$ is the matching where $m(i)$ is the match of agent $i$. Then, let $T_m(t)$ be the number of times matching $m$ is played by time $t$. For a matching $m$ to be played at time $t$ it must be stable according to the current preference rankings of the agents and the fixed rankings of the arms, i.e. according to $\prefsagent{i}{t}$ for all $\player{i} \in \players$ and $\pi_j$ for all $\arm{j} \in \arms$. We call such matchings \emph{achievable}.  We say a matching is \emph{truly stable} if it is stable according to the true preferences induced by the mean rewards of the arms. For agent $\player{i}$ and arm $\player{\ell}$ we consider the set $M_{i,\ell}$  of non-truly stable, achievable matchings $m$ such that $m(i) = \ell$. Let $\pessgap{i}{\ell}= \meanreward{i}{\pessmatch{i}} - \meanreward{i}{\ell}$. 

Then, since any truly-stable matching yields regret smaller or equal than zero for all agents,  we can upper bound the regret of agent $i$ as follows:
\begin{align}
  \label{eq:ucb_regret_decomposition}
\pessregret{i}(\horizon) \leq \sum_{\ell \colon \pessgap{i}{\ell} > 0} \pessgap{i}{\ell} \left( \sum_{m \in M_{i,\ell}} \EE T_m(\horizon)\right). 
\end{align}
For any matching $m$ that is non-truly stable there must exist an agent $\player{j}$ and an arm $\arm{k}$, different from arm $m(j)$, such that the pair $(\player{j}, \arm{k})$ is a \emph{blocking pair} according to the true preferences $\mu$, i.e. $\meanreward{j}{k} > \meanreward{j}{m(j)}$ and arm $\arm{k}$ is either unmatched or $\prefarm{k}{j} < \prefarm{k}{m^{-1}(k)}$. We say the triplet $(\player{j}, \arm{k}, \arm{k^\prime})$ is blocking when $\player{j}$ is matched with $\arm{k^\prime}$ and the pair $(\player{j}, \arm{k})$ is blocking.
Let $\blockedmatches{j}{k}{k^\prime}$ be the set of all matches blocked by the triplet $(\player{j}, \arm{k}, \arm{k^\prime})$. Given a set $S$ of matchings, we say a set $Q$ of triplets $(\player{j}, \arm{k}, \arm{k^\prime})$ is a \emph{cover} of $S$ if
\begin{align*}
\bigcup_{(\player{j}, \arm{k}, \arm{k^\prime}) \in Q} \blockedmatches{j}{k}{k^\prime} \supseteq S.
\end{align*}
Let $\calC(S)$ denote the set of covers of $S$. Also, let $\gaptrip{j}{k}{k^\prime} = \meanreward{j}{k} - \meanreward{j}{k^\prime}$. Now we state our result.  

\begin{theorem}
  \label{thm:meta}
  When all agents rank arms according to the upper confidence bounds \eqref{eq:ucb} and submit their preferences to the Gale-Shapley Platform, the regret of agent $\player{i}$ up to time $\horizon$ satisfies 
	\begin{align*}
	\pessregret{i}(\horizon) \leq \sum_{\ell \colon \pessgap{i}{\ell} > 0} \pessgap{i}{\ell} \left[\min_{Q \in \calC(M_{i, \ell})} \sum_{(\player{j}, \arm{k}, \arm{k^\prime}) \in Q} \left(5 + \frac{6 \log (n)}{\gaptrip{j}{k}{k^\prime}^2}\right) \right].
	\end{align*}
\end{theorem}

Theorem~\ref{thm:meta} offers a problem-dependent $\Ocal(\log(n))$ upper bound guarantee on the agent-pessimal stable regret of each agent $\player{i}$. Similarly to the case of centralized ETC, the regret of one agent depends on the suboptimality gaps of other agents. However, we saw in Section~\ref{sec:centralized_etc} that centralized ETC achieves $\Ocal(\log(n))$ agent-optimal stable regret, a stronger notion of regret.  Example~\ref{example:pessimal_regret} shows that centralized UCB cannot yield sublinear agent-optimal stable regret in general.   

\begin{example}[Centralized UCB does not achieve sublinear agent-optimal stable regret]
  \label{example:pessimal_regret}
	Let $\players = \{ \player{1}, \player{2}, \player{3} \}$ and $\arms = \{ a_1, a_2, a_3 \}$, with true preferences given by:
	\begin{align*}
	&\player{1}: \arm{1} \succ \arm{2} \succ \arm{3} &\quad \arm{1}: \player{2} \succ \player{3} \succ \player{1} \\
	&\player{2}: \arm{2} \succ \arm{1}  \succ \arm{3} &\quad \arm{2}: \player{1} \succ \player{2} \succ \player{3} \\
	&\player{3}: \arm{3} \succ \arm{1} \succ \arm{2} &\quad \arm{3}: \player{3} \succ \player{1} \succ \player{2}.
	\end{align*}
        The agent-optimal stable matching is $(\player{1}, \arm{1})$, $(\player{2}, \arm{2})$, $(\player{3}, \arm{3})$. When $\player{3}$ incorrectly ranks $\arm{1} \succ \arm{3}$ and the other two players submit their correct rankings, the Gale-Shapley Platform outputs the matching $(\player{1}, \arm{2})$, $(\player{2}, \arm{1})$, $(\player{3}, \arm{3})$. In this case $\player{3}$ will never correct their mistake because they never get matched with $\arm{1}$ again, and hence their upper confidence bound for $\arm{1}$ will never shrink. Figure~\ref{fig:three} illustrates this example; the optimal regret for $\player{1}$ and $\player{2}$ is seen to be linear in $\horizon$.
        % When the gap between $\arm{1}$ and $\arm{3}$ for $\player{3}$ is small, $\player{3}$ will make the mistake  $\arm{1} \succ \arm{3}$ with high probability. 
\end{example}

\begin{proof}[Proof of Theorem~\ref{thm:meta}]

Let $\pullstrips{j}{k}{k^\prime}{\horizon}$ be the number of times agent $\player{j}$ pulls arm $\arm{k^\prime}$ when the triplet $(\player{j}, \arm{k}, \arm{k^\prime})$ is blocking the matching selected by the platform. Then, by definition
\begin{align}
\label{eq:count_bmatches}
\sum_{m \in \blockedmatches{j}{k}{k^\prime}} T_m(\horizon) = \pullstrips{j}{k}{k^\prime}{\horizon}.
\end{align}
By the definition of a blocking triplet we know that if $\player{j}$ pulls  $\arm{k^\prime}$ when $(\player{j}, \arm{k}, \arm{k^\prime})$ is blocking, they must have a higher upper confidence bound for $\arm{k^\prime}$ than for $\arm{k}$. In other words, we are trying to upper bound the expected number of times the upper confidence bound on $\arm{k^\prime}$ is higher than that of the better arm $\arm{k}$ when we have the guarantee that each time this event occurs $\arm{k^\prime}$ is successfully pulled. Therefore, standard analysis for the single agent UCB \citep[e.g.,][Chap. 2]{Bubeck12regret} shows that 
\begin{equation}\label{eqn:ucb}
\EE L_{j,k,k^\prime}(\horizon) \leq 5 + \frac{6 \log(n)}{\gaptrip{j}{k}{k^\prime}^2}.
\end{equation} The conclusion follows from equations \eqref{eq:ucb_regret_decomposition} and \eqref{eq:count_bmatches}.
\end{proof}
 
To better understand the guarantee of Theorem~\ref{thm:meta} we consider two examples in which the markets have a special structure which enables us to simplify the upper bound on the regret. Moreover, in Corollary~\ref{cor:worst_case} we consider the  a worst case upper bound over possible coverings of matchings. 

\begin{example}[Global preferences]
  \label{example:global}
  Let $\players = \{\player{1}, \cdots, \player{\numagents}\}$ and $\arms=\{\arm{1}, \cdots, \arm{\numarms}\}$. We assume the following preferences: $ \player{i}: a_1 \succ \cdots \succ a_\numarms$ and $ \arm{j}: \player{1} \succ \cdots \succ \player{\numagents}$. In other words all agents have the same ranking over arms, and all arms have the same ranking over agents. Hence, the unique stable matching is $(\player{1}, a_1)$, $(\player{2}, a_2)$, \ldots, $(\player{\numagents}, a_\numagents)$. Moreover, for any $\player{i}$ and $\arm{\ell}$ we can cover the set of matchings $M_{i, \ell}$ with the triplets $(\player{i}, \arm{k}, \arm{\ell})$ for all $k$ with $1 \leq k \leq i$. Then, Theorem~\ref{thm:meta} implies \eqref{eq:global_bound} once we observe that $\gaptrip{i}{k}{\ell} \geq \pessgap{i}{\ell}$ for all $k \leq i$. 
  \begin{align}
    \label{eq:global_bound}
\pessregret{i}(\horizon) \leq 5 i \sum_{\ell = i + 1}^\numarms \pessgap{i}{\ell} + \sum_{\ell = i + 1}^\numarms \frac{6 i \log(\horizon)}{\pessgap{i}{\ell}}. 
  \end{align} 
 Figure~\ref{fig:global} illustrates this example empirically, displaying the optimal (also pessimal) regret of 5 out of 20 agents. The 1st-ranked agent has sublinear regret, consistent with \eqref{eq:global_bound}, while the 20th-ranked agent has negative regret and our upper bound is indeed 0.
\end{example}

\begin{example}[Unique pairs]
  Let  $\players = \{\player{1}, \cdots, \player{\numagents}\}$ and $\arms=\{\arm{1}, \cdots, \arm{\numagents}\}$ and assume that agent $\player{i}$ prefers arm $\arm{i}$ the most and that arm $\arm{i}$ prefers agent $\player{i}$ the most. Therefore, the unique stable matching is $(\player{1}, a_1)$, $(\player{2},a_2)$, \ldots, $(\player{\numagents}, a_\numagents)$. Then, we can cover each set $M_{i, \ell}$ with the triplet $(\player{i}, \arm{i}, \arm{\ell})$. Therefore, Theorem~\ref{thm:meta} implies \eqref{eq:unique_pairs_bound}; note that the right-hand side is identical to the guarantee for single agent UCB: 
  \begin{align}
    \label{eq:unique_pairs_bound}
\pessregret{i}(\horizon) \leq 5 \sum_{\ell \neq i}^\numarms \pessgap{i}{\ell} + \sum_{\ell \neq i}^\numagents \frac{6 \log(\horizon)}{\pessgap{i}{\ell}}. 
 \end{align}
\end{example}

 \begin{corollary}
          \label{cor:worst_case}
Let  $\Delta = \min_i \min_{j, j^\prime} |\meanreward{i}{j} - \meanreward{i}{j^\prime}|$. When all players follow the centralized UCB method, the regret of $\player{i}$ can be upper bounded as follows 
  \begin{align*}
    \pessregret{i}(\horizon) \leq \max_{\ell} \Delta_{i, \ell} \left(6 \numagents \numarms^2 + 12 \frac{\numagents \numarms \log (\horizon) }{\Delta^2}\right).
  \end{align*}
\end{corollary}
%\begin{proof}
\noindent
\emph{Proof} We consider the  covering $(j, k, k^\prime)$ composed of all possible triples with $\meanreward{j}{k} > \meanreward{j}{k^\prime}$. Then, Theorem~\ref{thm:meta} implies the result because
$
\sum_{k^\prime \colon \meanreward{j}{k^\prime} < \meanreward{j}{k}} \frac{1}{\Delta_{j,k,k^\prime}^2} \leq \sum_{\ell = 1}^\numarms \frac{1}{\ell^2 \Delta^2} \leq \frac{2}{\Delta^2}. 
$

\begin{figure}[t]
	\centering
	\begin{subfigure}[t]{\basefigwidth\textwidth}
		\centerline{\includegraphics[width=\columnwidth]{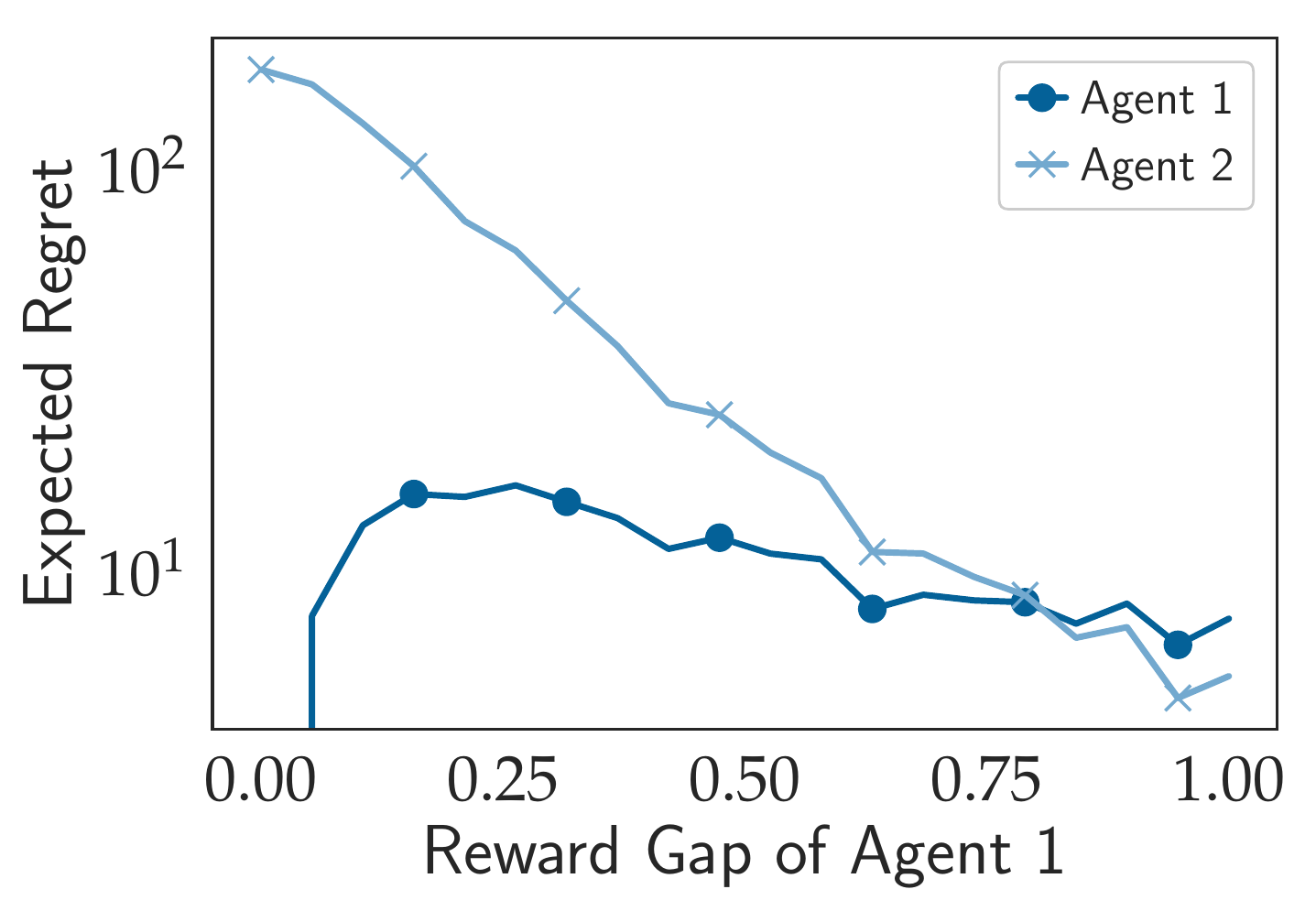}}
		\subcaption{\small Example~\ref{example:delta2}}
		\label{fig:gap}
	\end{subfigure}
	\begin{subfigure}[t]{\basefigwidth\textwidth}
		\centerline{\includegraphics[width=\columnwidth]{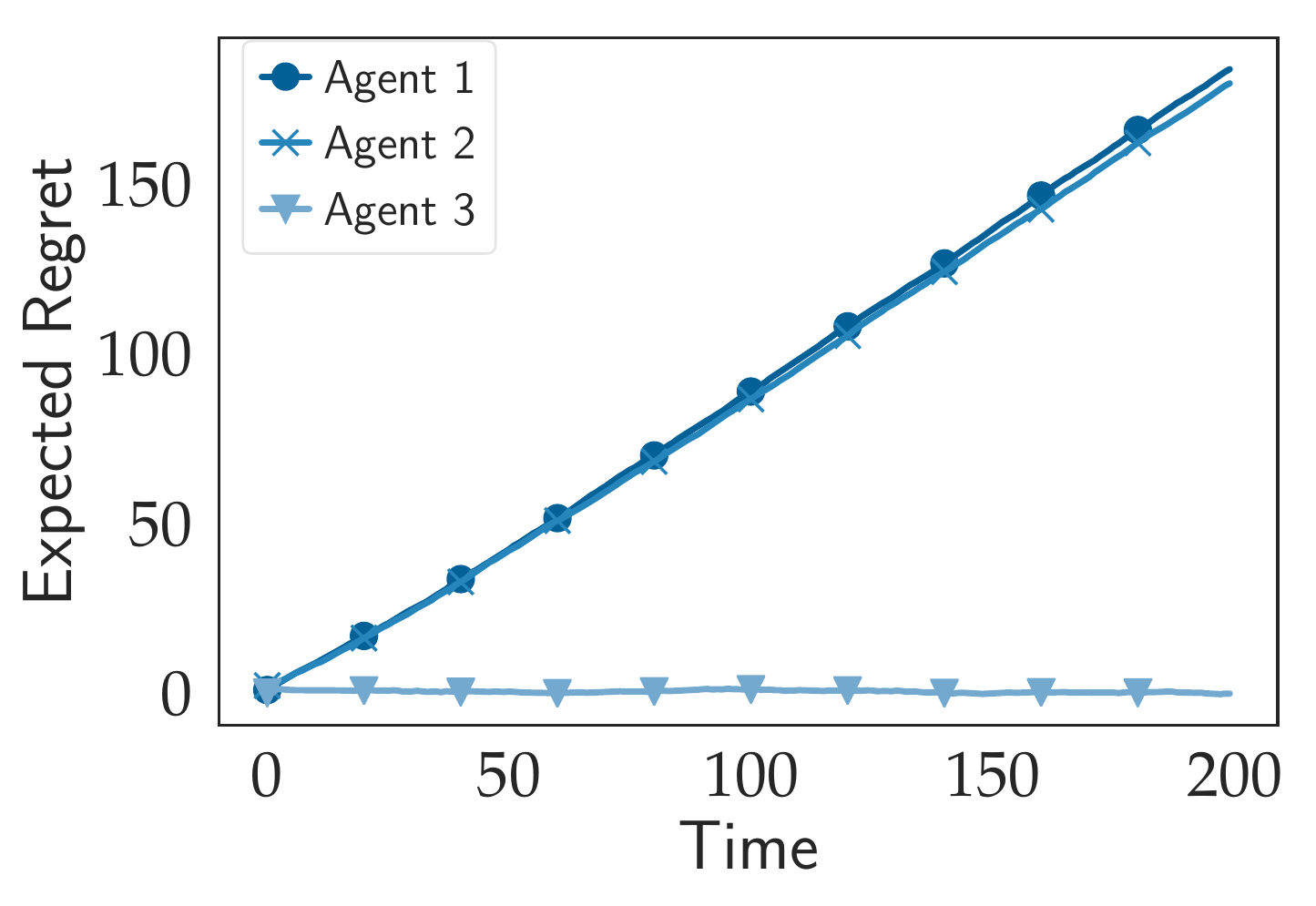}}
		\subcaption{\small Example~\ref{example:pessimal_regret}}
		\label{fig:three}
	\end{subfigure}
	\begin{subfigure}[t]{\basefigwidth\textwidth}
		\centerline{\includegraphics[width=\columnwidth]{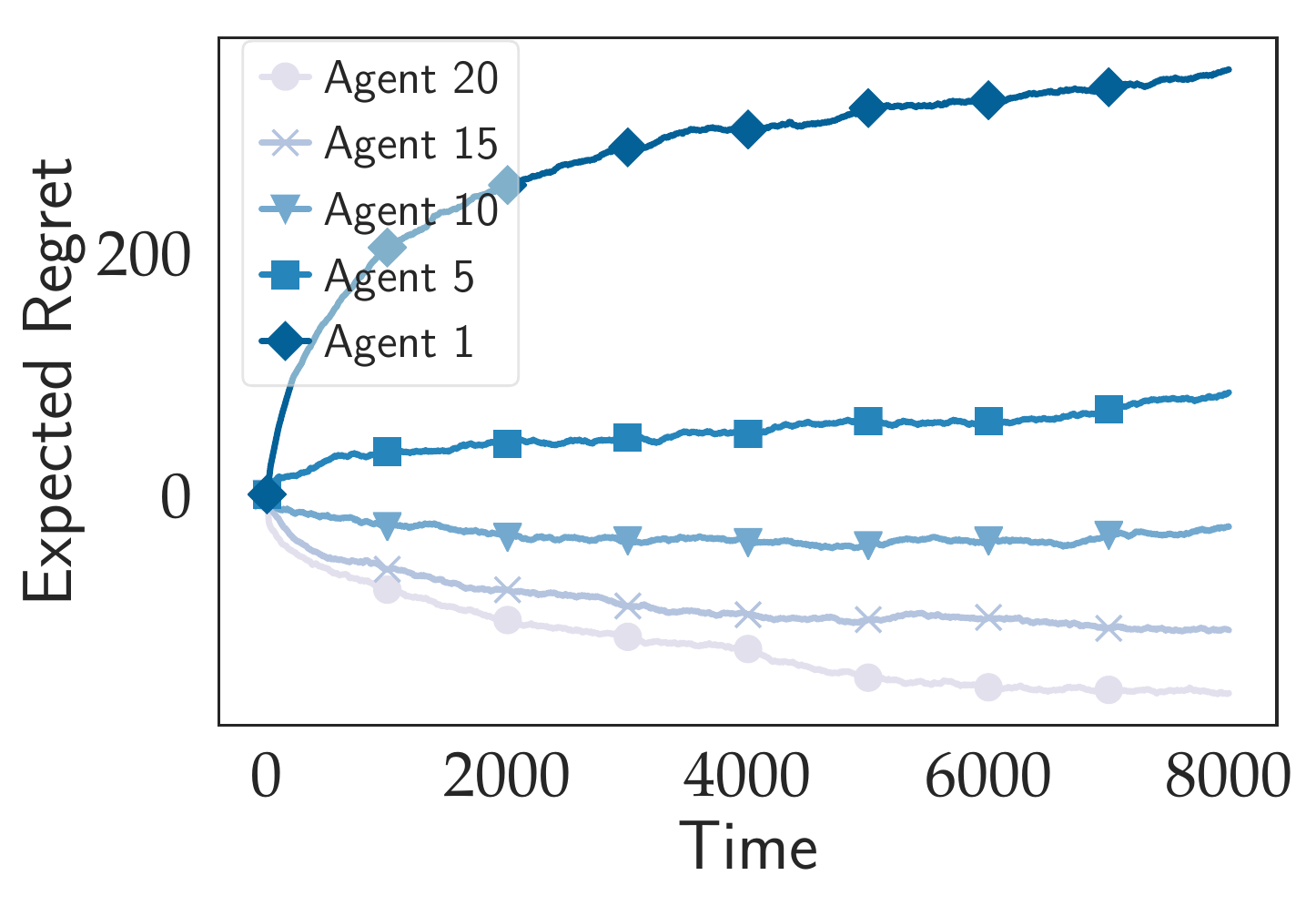}}
		\subcaption{\small Example~\ref{example:global}}
		\label{fig:global}
	\end{subfigure}
	\caption{\small The empirical performance of centralized UCB in the  settings described in Examples \ref{example:delta2}, \ref{example:pessimal_regret}, and  \ref{example:global}. See Appendix~\ref{app:exp_details} for the experimental details.} 
	\label{fig:ucb}
        \vspace{-15pt}
\end{figure}

\subsection{Honesty and Strategic Behavior}
\label{sec:honesty}

Classical results show that in the agent-proposing GS algorithm, no single agent can improve their match by misrepresenting their preferences, assuming that the other agents and arms submit their true preferences~\citep{roth1982economics,dubins1981mac}. The result generalizes to coalition of agents. Moreover, when there is a unique stable matching, the Dubins-Freedman Theorem says that no arms or agents can benefit from misrepresenting their preferences~\citep{dubins1981mac}.

The ETC Platform does not allows agents to choose which arms to explore. In this case, the classical results on honesty in agent-proposing GS apply; the agents are incentivized to submit the rankings according to their current mean estimates. 
When agents have some degree of freedom to explore over multiple rounds, it is no longer clear if any agents, or arms, can benefit from misrepresenting their preferences in some of the rounds.
In general, one agent's preferences can influence not only the matches of other agents, but also their reward estimates. One might be able to improve their regret by capitalizing on the ranking mistakes of other agents. The possibilities for long-term strategic behavior are more diverse than in the single-round setting.

 We now show that when all agents except one submit their UCB-based preferences to the GS Platform, the remaining agent has an incentive to also submit preferences based on their UCBs, so long as they do not have multiple stable arms.

First, we establish the following lemma, which is an upper bound on the expected number of times the remaining agent can pull an arm that is better than their optimal match, regardless of what preferences they might have submitted to the platform.

 \begin{lemma}\label{lem:better-than-optimal}
 	Let $T_l^{i}(n)$ be the number of times an agent $i$ pulls an arm $l$ such that the mean reward of $l$ for $i$ is greater than $i$'s optimal match. Then
 	\begin{equation}
 	\E[T_l^{i}(n)] \leq \min_{Q \in \calC(M_{i, \ell})} \sum_{(j,k,k^\prime) \in Q} \left(5 + \frac{6 \log (n)}{\gaptrip{j}{k}{k^\prime}^2}\right)
 	\end{equation}
 \end{lemma}
 \begin{proof}
 	If agent $i$ is matched with arm $l$ in any round, the matching $m$ must be unstable according true preferences. We claim that there must exist a blocking triplet $(j,k,k')$ where $j \neq i$.
 	
 	Arguing by contradiction, we suppose otherwise, that all blocking triplets in $m$ only involve agent~$i$. By Theorem 4.2 in \citet{hernan95paths}, we can go from the matching $m$ to a $\mu$-stable matching, by iteratively \emph{satisfying} block pairs in a `gender consistent' order $O$. To satisfy a blocking pair $(k,j)$, we break their current matches, if any, and match $(k,j)$ to get a new matching. Doing so, agent $i$ can never get a worse match than $l$ or become unmatched as the algorithm proceeds, so the matching remains unstable---a contradiction. Hence there must exist a $j \neq i$ such that $j$ is part of a blocking triplet in $m$.
 	In particular, agent $j$ must be submitting its UCB preferences.
 	
 	The result then follows from the identity \[\E[T_l^{i}(n)] = \sum_{m \in M_{i,\ell}} \EE T_m(\horizon),\] and Equation \ref{eqn:ucb}
 \end{proof}
 
Lemma \ref{lem:better-than-optimal} directly implies the following lower bound on the remaining agent's optimal regret.
\begin{proposition}\label{prop:ucb_lower}
	Suppose all agents other than $\player{i}$ submit preferences according to the UCBs \eqref{eq:ucb} to the GS Platform. Then the following upper bound on agent $i$'s optimal regret holds:
	\begin{equation}
	\optregret{i}(n) \ge \sum_{\ell \colon\optgap{i,l} <0} \optgap{i,l} \left[\min_{Q \in \calC(M_{i, \ell})} \sum_{(j,k,k^\prime) \in Q} \left(5 + \frac{6 \log (n)}{\gaptrip{j}{k}{k^\prime}^2}\right) \right].
	\end{equation}
\end{proposition}
Therefore, there is no sequence of preferences that an agent can submit to the GS Platform that would give them negative optimal regret greater than $\Ocal(\log n)$ in magnitude. When there is a unique stable matching, Proposition \ref{prop:ucb_lower} shows that no agent can gain significantly in terms of stable regret by submitting preferences other than their UCB rankings.
%When there exist multiple stable matchings, however, Proposition~\ref{prop:ucb_lower} leaves open the question of whether any agent can submit a sequence of preferences that achieves super-logarithmic negative \emph{pessimal} regret for themselves, when all other agents are playing their UCB preferences. In other words, can an agent do significantly better than its pessimal stable arm, by possibly deviating from their UCB rankings? 

When there exist multiple stable matchings, however, Proposition~\ref{prop:ucb_lower} leaves open the question of whether any agent can submit a sequence of preferences that achieves super-logarithmic negative \emph{pessimal} regret for themselves, when all other agents are playing their UCB preferences. In other words, can an agent do significantly better than its pessimal stable arm, by possibly deviating from their UCB rankings? This is an interesting direction for future work.

 \section{Decentralized Explore-Then-Commit}\label{sec:decent-etc}

 In the decentralized setting of the matching bandits problem, we propose a simple algorithm based on Explore-Then-Commit (ETC) that can achieve low agent-\emph{optimal} regret, albeit at a suboptimal rate. A key observation behind this algorithm is that the Gale-Shapley algorithm can be implemented with simultaneous proposing \cite[see e.g.][Theorem 1]{roth07deferred}, hence a central platform is not necessary for the agents to reach a stable matching. We analyze this simple algorithm in order to motivate the search for more efficient algorithms in the decentralized setting.
 
 \paragraph{Description of the algorithm}
 There are three stages.  Stage 1 has $H\numarms$ rounds. In stage 1 (``Exploration''), every $K$ rounds, each agent independently samples a random permutation of arms and attempts arms in that order. Agents update the respective sample means of the arms only if the pull was successful. In stage 2 (``Simultaneous Proposing GS"), each round each agent attempts the arm with the highest sample mean that they haven't had a conflict on in Stage 2. Stage 2 continues for $\numagents$ rounds. In stage 3 (``Exploitation''), every agent keeps pulls the last arm they pulled successfully in stage 2.
 
 In the following result, proved in Appendix \ref{app:pf-decent-etc}, we analyze the regret of the decentralized ETC.
 
 \begin{proposition}[Regret bound for decentralized ETC]\label{prop:decent-etc}
 	Consider Decentralized ETC with stage 1 lasting $HK$ rounds. Let $\optgap{i,j} $ , $\optgap{i,\max} $ and $\Delta $ be defined as before in Theorem \ref{thm:centralized_etc}. Let $\rho_{\numagents,\numarms} := \left( 1 - \frac{1}{\numarms} \right)^{\numagents-1}$.
 	The expected agent-optimal regret of player $\player{i}$ is upper bounded by
 	\begin{align}
 	\optregret{i}(\horizon) \leq HK\meanreward{i}{\optmatch{i}}+ (\horizon - H \numarms) \optgap{i, \max} \numagents \numarms \left( 2\exp\left(-\frac{H\rho_{\numagents,\numarms}^2}{2}\right) +  \exp\left( - \frac{H\rho_{\numagents,\numarms}\Delta^2 }{8}\right) \right).
 	\end{align}
 \end{proposition}

\input{related_work}

\section{Discussion} 

We have proposed a new model for dynamic matching in markets under uncertain preferences. 
The model blends learning and competition, and captures two desirable notions: stability 
and sample efficiency.  We presented two natural algorithms which combine classical 
ideas from multi-armed bandits and stable matching.  Our focus in the current paper was 
the centralized UCB method, which we proved enjoys small regret for each player and 
ensures that the market converges quickly to a stable configuration. 

There are many additional questions that can be studied in this model, including 
problems with incomplete information, decentralized matching protocols and shared
reward structures.  We have already seen that the uncertainty of one agent can depress 
the long-term utility of other agents, and we expect to uncover other interactions 
between learning and strategic decision making in this model.

\paragraph{Acknowledgements} The authors thank Alekh Agarwal, S\'ebastien Bubeck, Ofer Dekel, Moritz Hardt, Kevin Jamieson, and Yishay Mansour for helpful comments. This work was supported in part by the Mathematical Data Science program of the Office of Naval Research under grant number N00014-18-1-2764 and by the Army Research Office under grant number W911NF-17-1-0304.

\bibliographystyle{abbrvnat}
\bibliography{recmkt}

\begin{appendix}
	
	\section{Experimental Details}\label{app:exp_details} 
	
	\paragraph{Figure~\ref{fig:gap}.} This figure represents an empirical evaluation of Example~\ref{example:delta2}. In this setting, there are two agents and two arms. Player $\player{2}$ receives Gaussian rewards from the arms $\arm{1}$, $\arm{2}$ with means $0$ and $1$ respectively and variance $1$.  Player $\player{1}$ receives Gaussians rewards $\Delta$ and $0$ from the arms $\arm{1}$ and $\arm{2}$. Both arms prefer $\player{1}$ over $\player{2}$. Figure~\ref{fig:gap} shows the regret of each agent as a function of $\Delta$ when we run centralized UCB with horizon $400$ and average over $100$ trials. 
	
	\paragraph{Figure~\ref{fig:three}.} This figure represents an empirical evaluation of Example~\ref{example:delta2}. The rewards of the arms for each agent are Gaussian with variance $1$. The have mean rewards of the arms are set so that the preference structure shown in Example \ref{example:delta2} is satisfied. For agents $\player{1}$ and $\player{2}$, the gap in mean rewards between consecutive arms is $1$. For agent $\player{3}$ the gap in mean reward between arms $\arm{1}$ and $\arm{3}$ is $0.05$.  Figure~\ref{fig:three} shows the performance of centralized UCB, averaged over $100$ trials, as a function of the horizon.
	
	\paragraph{Figure~\ref{example:global}.} This figure represents an empirical evaluation of Example~\ref{example:global} when there are $20$ agents and $20$ arms. The rewards of the arms are Guassian with variance $1$. The mean reward gap between consecutive arms is $0.1$.  Figure~\ref{fig:three} shows the performance of centralized UCB, averaged over $50$ trials, as a function of the horizon.

	\section{Proof of Proposition \ref{prop:decent-etc}}\label{app:pf-decent-etc}
	
	\begin{proof}
		
		Suppose stage $1$ lasts $HK$ rounds.
		Fix the agent. 
		For any particular attempt on arm $i$, let $Y_i$ denote the event of a successful pull. The probability of a successful pull is bounded from below by the probability of a successful pull for an agent that is the least preferred by the arm attempted.
		\[ p_i:= \PP\{Y_i = 1\} \geq \left( 1 - \frac{1}{\numarms} \right)^{\numagents-1}=: \rho_{\numagents,\numarms}\]
		
		Let $T_i$ be the number of times arm $i$ was pulled in stage 1. By the independence of the random permutations sampled, we have that $T_i \sim \text{Binomial}(H, p_i)$. We can use a standard tail bound:
		\[ \PP\{ T_i \leq t \} \leq \exp\left(-2\frac{(Hp_i-t)^2}{H}\right) \]
		
		Stage 2 gives rise to a matching that is stable according to the order of the average rewards ($\hat{R}^i_j(T)$), after $N$ rounds. This is essentially the Gale-Shapley algorithm but with simultaneous proposals.
		
		Now we bound the probability that this matching is agent optimal according to the true preferences. If any agent $j$ ranks arm $k$ and arm $k'$ wrongly, we must have $\hat{R}^k_j(H) > \hat{R}^{k'}_j(H) $ but $\meanreward{j}{k'} > \meanreward{j}{k} $. Therefore, we may bound the probability of a blocking pair using the sub-Gaussianity of $\hat{R}^k_j(H) - \hat{R}^{k'}_j(H) $. Let $A_{k,k'}$ denote the event $\{ (\hat{R}^{k'}_j -\hat{R}^{k}_j  ) - (\meanreward{j}{k'} - \meanreward{j}{k}) \leq -(\meanreward{j}{k'} - \meanreward{j}{k}) \}$. 
		
		\begin{align*}
		\PP(A_{k,k'}) &= \sum_{j,j' < H} \PP(A_{k,k'} \cap T_k = j \cap T_{k'} = j') \\
		&=  \sum_{j \wedge j' < h} \PP(A_{k,k'} \cap T_k = j \cap T_{k'} = j') + \sum_{j \wedge j' \geq h} \PP(A_{k,k'} \cap T_k = j \cap T_{k'} = j') \\
		&\leq \PP(T_k < h) + \PP(T_{k'} < h) + \sum_{j \wedge j' \geq h}  \PP(A_{k,k'} \mid  T_k = j , T_{k'} = j') \cdot \PP( T_k = j , T_{k'} = j') \\
		&\leq \PP(T_k < h) + \PP(T_{k'} < h) + \sum_{j \wedge j' \geq h}  \exp\left( - \frac{(j \wedge j')(\gap{k}{k'})^2 }{4}\right)  \PP( T_k = j , T_{k'} = j') \\
		&\leq  \PP(T_k < h) + \PP(T_{k'} < h) + \exp\left( - \frac{h(\gap{k}{k'})^2 }{4}\right).
		\end{align*}
		
		Choosing $h = \frac{1}{2}Hp_i$ gives \[\PP(A_{k,k'}) \leq 2\exp\left(-\frac{Hp_i^2}{2}\right) +  \exp\left( - \frac{Hp_i(\gap{k}{k'})^2 }{8}\right) \le  2\exp\left(-\frac{H\rho_{\numagents, \numarms}^2}{2}\right) +  \exp\left( - \frac{H\rho_{\numagents, \numarms}\Delta^2 }{8}\right)\]
		
		By Lemma \ref{lem:opt_valid_rankings}, we only have to consider $k' = \optmatch{i}$ and $k$ such that $\optgap{i,k} > 0$, so there are at most $K$ such pairs, for each agent.
		
	\end{proof}

\end{appendix}

\end{document}

%% file: notation.tex
\newcommand{\E}{\mathbb{E}}
\newcommand{\PP}{\mathbb{P}}

\newcommand{\calC}{\mathcal{C}}

\newcommand{\indi}{\mathds{1}}

\newcommand{\player}[1]{p_{#1}}
\newcommand{\arm}[1]{a_{#1}}
 
\newcommand{\players}{\mathcal{N}}
\newcommand{\arms}{\mathcal{K}}

\newcommand{\prefarm}[2]{\pi_{#1}(#2)}
\newcommand{\prefagent}[3]{\widehat{r}_{{#1}, {#2}}({#3})} % agent, arm, time

\newcommand{\prefsarm}[1]{\pi_{#1}}
\newcommand{\prefsagent}[2]{\widehat{r}_{{#1}, {#2}}} % agent, time

\newcommand{\meanreward}[2]{\mu_{{#1}}({#2})} % agent, arm

\newcommand{\hatmeanreward}[3]{\widehat{\mu}_{{#1}, #2}(#3)} % agent, arm, time

\newcommand{\reward}[3]{X_{{#1}, {#2}}({#3})}

\newcommand{\actions}[1]{m_{{#1}}}

\newcommand{\optmatches}{\overline{m}}
\newcommand{\pessmatches}{\underline{m}}

\newcommand{\optmatch}[1]{\overline{m}(#1)}
\newcommand{\pessmatch}[1]{\underline{m}(#1)}

\newcommand{\action}[2]{m_{{#1}}({#2})}

\newcommand{\numarms}{K}
\newcommand{\numagents}{N}
\newcommand{\horizon}{n}
\newcommand{\optgap}[1]{\overline{\Delta}_{#1}}
\newcommand{\pessgap}[2]{\underline{\Delta}_{#1, #2}}
\newcommand{\optregret}[1]{\overline{R}_{#1}}
\newcommand{\pessregret}[1]{\underline{R}_{#1}}

\newcommand{\gap}[2]{\Delta_{{#1}, {#2}}}
\newcommand{\gaptrip}[3]{\Delta_{{#1}, {#2}, {#3}}} %agent and two arms for which the gap is computed

\newcommand{\ucb}[3]{u_{{#1}, {#2}} ({#3})} % agent, arm, time

\newcommand{\pullstrips}[4]{L_{{#1}, {#2}, {#3}}({#4})} % agent, better arm, arm selected, horizon
\newcommand{\blockedmatches}[3]{B_{#1, #2, #3}} % agent, better arm, arm selected

%% file: abstract.tex
% !TeX root = main.tex 
Stable matching, a classical model for two-sided markets, has long been studied with little consideration for how each side's preferences are \emph{learned}. With the advent of massive online markets powered by data-driven matching platforms, it has become necessary to better understand the interplay between learning and market objectives. We propose a statistical learning model in which one side of the market does not have a priori knowledge about its preferences for the other side and is required to learn these from stochastic rewards. Our model extends the standard multi-armed bandits framework to multiple players, with the added feature that arms have preferences over players. We study both centralized and decentralized approaches to this problem and show surprising exploration-exploitation trade-offs compared to the single player multi-armed bandits setting. 

%% file: related_work.tex
% !TeX root = main.tex 
\section{Related work}\label{sec:relwork}

%This work uses the stochastic multi-armed bandit (MAB) problem to model exploration-exploitation tradeoffs in online two-sided matching. 
Since its introduction by \citet{Thompson33} in 1933, the stochastic multi-armed bandit (MAB) problem
has inspired a rich body of work spanning different settings, 
algorithms, and guarantees \citep{Lai1985asymp,Bubeck12regret,Szepesvari2018algs}. 
%In the classical problem, a single decision maker pulls an arm $I_t \in \{1, \cdots, K\}$ at each round $t$, and receives a reward $X_{I_t, t}$, which is a random variable with mean $\mu_{I_t}$. The goal is to minimize expected \emph{regret}, which is the expected difference between their total reward and the reward of choosing the  best action in hindsight. Among many algorithms that attain the optimal regret of $\Ocal(\log(\horizon))$, the UCB algorithm is a classical one that applies ``optimism in the face of uncertainty'' by choosing the arm with the higher upper confidence bound at each time step 
%\cite{Lai1985asymp,agrawal95, Katehakis1995SequentialCF,  Auer:2002}.

There has been recent interest in the MAB literature in problems with multiple, interacting 
players \citep{pmlr-v49-cesa-bianchi16,Shahrampour17}. In one popular 
formulation known as \emph{bandits with collision}, multiple players choose from the same 
set of arms, and if two or more players choose the same arm, no reward is received by any 
player~\citep[e.g.][]{Liu10distributed, Anandkumar:2011distri, avner2014, rosenski16}. This differs from our formulation, where arms 
have preferences and the most preferred player receives a reward, while the other player
selecting the arm do not. For the stochastic bandits with collision problem where agents have the same preferences for arms, works differ by the extent to which agents are allowed to communicate and to see collisions---the most challenging setting is where there is no communication and no collision information \citep{lugosi18multi}. \citet{bubeck2019multiplayer} delineated
the optimal rates for the non-stochastic version of the problem.

A variant of this problem is where agents have different preferences over arms. For this problem, \citet{Bistriz2018}'s algorithm approximately finds the maximum matching of players to arms with $\Ocal(\log(\horizon)^{2+\kappa})$ regret. However, stable matching does not reduce to maximum matching in general, so such guarantees do not apply to matching with two-sided preferences.

The two-sided matching problem has also been studied in sequential settings. \citet{Das2005two} 
proposed an empirical study of a two-sided matching problem where both sides of the market have 
uncertain preferences  \citet{Johari2017matching} studied a sequential matching problem in
which the market participants are modeled with arrival processes.

\citet{Ashlagi2017communication} considered the communication and preference learning cost of stable matching. Their model formulates preference learning as querying a noiseless choice function, rather than obtaining noisy observations of one's underlying utility. Different players can query their choice functions independently; 
hence congestion in the preference learning stage is not captured by this model. In many markets, 
obtaining information about the other side of the market itself can lead to congestion and thus 
the need for strategic decision. For example, \citet[][chap. 10]{roth1990two}, note that graduating 
medical students go to interviews to ascertain their own preferences for hospitals, but the interviews 
that a student can schedule are limited. Our model begins to capture such tradeoffs by introducing 
statistical uncertainty in the preferences of one side of the market and providing a natural mode 
of interaction between the learning agents.

%Finally, a relevant line of work is on building recommendation systems with capacity constraints and two sided preferences. \citet{goswami2017Recommend} studied recommendation for markets with two sided preferences, also assuming that data for preferences and matches are collected offline. They do not consider notions of market efficiency and instead uses the probability of predicting the final match correctly as the only metric. \citet{Christakopoulou2017Recommend} proposed a heuristic for providing recommendations under capacity constraints, though without two-sided preferences; they also assume that ratings data is collected offline.